\documentclass{article}

\usepackage{arxiv}

\usepackage[toc,page]{appendix}
\usepackage[utf8]{inputenc} % allow utf-8 input
\usepackage[T1]{fontenc}    % use 8-bit T1 fonts
\usepackage{hyperref}       % hyperlinks
\usepackage{url}            % simple URL typesetting
\usepackage{booktabs}       % professional-quality tables
\usepackage{amsfonts}       % blackboard math symbols
\usepackage{nicefrac}       % compact symbols for 1/2, etc.
\usepackage{microtype}      % microtypography
\usepackage{lipsum}

\usepackage{amsmath,amssymb,bm,comment,environ,grffile,latexsym,verbatim,xspace} 

\makeatletter
\@ifpackageloaded{float}{}{
  \usepackage{floatrow}
  \floatsetup[table]{capposition=top}
}
\makeatother

\makeatletter
\@ifpackageloaded{cite}{}{
  \@ifpackageloaded{natbib}{}{\usepackage[sort&compress,numbers]{natbib}}
}
\makeatother
\makeatletter
\@ifpackageloaded{algorithmic}{}{\usepackage[noend]{algpseudocode}}
\@ifpackageloaded{algorithm}{}{\usepackage[ruled]{algorithm}}
\makeatother

\usepackage{graphicx} % more modern
\usepackage{grffile} % support suffix

\makeatletter
\@ifclassloaded{beamer}{
  \usepackage[compatibility=false]{caption}  
  }{
  \usepackage{caption}
}
\makeatother

\usepackage[labelformat=simple]{subcaption}
\usepackage{multirow}

\usepackage{tikz}
\usepackage{pgfplots}
\usetikzlibrary{shapes,shapes.arrows,shapes.gates.logic.US,trees,positioning,arrows,automata,decorations.pathmorphing,chains}
\usetikzlibrary{shadows,calc,backgrounds,patterns,fit}
\usetikzlibrary{matrix,arrows}
\usepackage{url}

\makeatletter
\@ifpackageloaded{hyperref}{}{
  \usepackage[colorlinks,linkcolor={red!80!black},citecolor={blue!80!black},urlcolor={blue!80!black}]{hyperref}
}

\@ifclassloaded{beamer}{
  \usepackage{amsfonts}
  }{
  \usepackage{times}
  \usepackage{mathrsfs,dsfont} % dsfont to support 
  \usepackage{mathbbol}
  \DeclareSymbolFontAlphabet{\mathbbl}{bbold}
  \usepackage{amsfonts}
  \DeclareSymbolFontAlphabet{\mathbb}{AMSb}
  \usepackage{enumitem}
  \setlist{noitemsep}
  \setlist[1]{noitemsep}
  \setlist[1]{nosep}

  \newlist{compactitem}{itemize}{3}
  \setlist[compactitem]{topsep=0pt,partopsep=0pt,itemsep=0pt,parsep=0pt}
  \setlist[compactitem,1]{label=\textbullet}
  \setlist[compactitem,2]{label=---}
  \setlist[compactitem,3]{label=*}

  \newlist{compactdesc}{description}{3}
  \setlist[compactdesc]{topsep=0pt,partopsep=0pt,itemsep=0pt,parsep=0pt}

  \newlist{compactenum}{enumerate}{3}
  \setlist[compactenum]{topsep=0pt,partopsep=0pt,itemsep=0pt,parsep=0pt}
  \setlist[compactenum,1]{label=\arabic*}
  \setlist[compactenum,2]{label=\alph*}
  \setlist[compactenum,3]{label=\roman*}
}
\makeatother

\usepackage{etex}  % For larger registers for latex
\usepackage{color}
\usepackage{mathtools}
\usepackage{ifthen}
\usepackage{etoolbox}

\usepackage{xr}
\makeatletter
\newcommand*{\addFileDependency}[1]{
        \typeout{(#1)}
        \@addtofilelist{#1}
        \IfFileExists{#1}{}{\typeout{No file #1.}}
}
\makeatother

\input{defs.tex}

\newcommand{\R}[1]{\Rcal_w\rbr{#1}}

\newcommand{\fig}{.}

\newcommand{\assistant}{{\textit{Assistant}}\xspace}
\newcommand{\boss}{{\textit{Boss}}\xspace}
\newcommand{\autoassist}{{\textit{AutoAssist}}\xspace}

\newcommand{\uniform}[1]{\mathtt{uniform}\rbr{#1}}
\ifx\proposition\undefined
\newtheorem{proposition}{Proposition}
\fi
\title{AutoAssist: A Framework to
Accelerate Training of Deep Neural Networks}

\author{
  Jiong Zhang\\
  University of Texas at Austin\\
  \texttt{zhangjiong724@utexas.edu} \\
   \And
  Hsiang-fu Yu \\
  Amazon\\
  \texttt{rofu.yu@gmail.com} \\
   \And
  Inderjit S. Dhillon \\
  University of Texas at Austin \& Amazon\\
  \texttt{inderjit@cs.utexas.edu}
}

\begin{document}
\maketitle
\begin{abstract}
Deep neural networks have yielded superior performance in many
applications; however, the gradient computation in a deep model with millions of
instances leads to a lengthy training process even with modern
GPU/TPU hardware acceleration.  In this paper, we propose \autoassist, a
simple framework to accelerate training of a deep neural network.
Typically, as the training procedure evolves, the amount of improvement in the current model by a
stochastic gradient update on each instance varies dynamically .
In \autoassist, we utilize this fact and design a simple {\em instance shrinking} operation, which is
used to filter out instances with relatively low marginal improvement to the
current model; thus the computationally intensive gradient computations are
performed on informative instances as much as possible. We prove that
the proposed technique outperforms vanilla SGD with existing importance
sampling approaches for linear SVM problems, and establish an $O(1/k)$
convergence for strongly convex problems. In order to apply the proposed
techniques to accelerate training of deep models, we propose to jointly train
a very lightweight \assistant network in addition to the original deep network referred to as \boss.
The \assistant network is designed to gauge the importance of a given
instance with respect to the {\em current} \boss such that a shrinking operation can
be applied in the batch generator. With careful design, we train the \boss and
\assistant in a nonblocking and asynchronous fashion such that
overhead is minimal. We demonstrate that
\autoassist reduces the number of epochs by $40\%$ for training a ResNet to reach the same test
accuracy on an image classification data set, and saves $30\%$ training time 
needed for a transformer model to yield the same BLEU scores on a translation
dataset.
\end{abstract}

\vspace{-0.2cm}
\section{Introduction}
Deep neural networks trained on a large number of instances have been 
successfully applied to many real world applications, such 
as~\cite{collobert2008unified,he2016deep} and~\cite{langkvist2014review}. Due to 
the increasing number of training instances and the increasing complexity of 
deep models, variants of (mini-batch) stochastic gradient descent (SGD) are 
still the most widely used optimization methods because of their simplicity and 
flexibility. In a typical SGD implementation, a batch of instances is 
generated either by using a random permuted order or a uniform sampler. Due to 
the complexity of deep models, the gradient calculation is usually extremely 
computationally intensive and requires powerful hardware (such as a GPU or TPU) to 
perform the entire training in reasonable time. At any given time 
of the learning, each instance has its own utility in terms of improving the 
current model. As a result, performing SGD updates on a batch of instances 
which are sampled/generated uniformly is suboptimal in terms of maximization 
of return-on-investment~(ROI) of GPU/TPU cycles. In this paper, we propose  
\autoassist, a simple framework to accelerate training deep models with an
\assistant that generates instances in a sequence which attempts to maximize the ROI.~\nocite{chang2017active,kim2018screenernet,bengio2015scheduled}

There have been attempts to improve the training speed of deep learning. 
In \cite{bengio2009curriculum}, curriculum learning~(CL), where ``easier'' instances are 
presented earlier than ``harder'' instances, was shown to be beneficial to the 
overall convergence; however, prior knowledge of the training set is required to 
rank the instances by its simplicity or difficulty. Self-paced 
learning~(SPL)~\cite{kumar2010self} is another attempt that infers the 
``difficulty'' of instances based on the corresponding loss value during  
training and decreases the sample probability of these difficult instances, 
which, however, ends up overly emphasizing the easier samples. However, the marginal 
gain of information provided by these easier instances decreases as the model gets 
better, which leads to slow convergence of SPL. 
%instances results in  
%smaller gradients and the easy instances will be less and less informative as 
%the model converges. 
%Jiang, Lu, \textit{ et al.} 
\cite{jiang2015self}~combined the above two ideas and proposed Self Paced Curriculum 
learning~(SPCL), which is formulated as an optimization 
problem which utilizes both prior knowledge and the loss values 
as the learning progresses.  
However SPCL relies on a manually chosen scheme function which has $O(N)$ 
parameters to be learned, where $N$ is number of instances in the dataset, which introduces a considerable overhead into the 
training process in terms of both 
time and space complexity. 
%We show that these two opposing 
%strategies are beneficial at different stage of training: relying on easy 
%instances can make the model converging fast at initial stage and robust 
%towards noises, however as the model converges easy instances are less 
%informative and to get to better model we need to focusing on hard examples.

In this paper, we propose \autoassist, a simple framework to accelerate 
training deep models  with an \assistant that generates instances in a sequence 
which attempts to maximize the ROI. In particular, the main model, referred to as the \boss, 
is trained with instances generated by a light-weight \assistant, which 
yields instances in a sequence that tries to maximize the ROI for the current \boss. In 
particular, the \assistant is designed in a way to adapt to the changes in 
\boss dynamically and asynchronously. 
%Time and space efficiency is another issue in these learning architectures.
%\cite{chang2017active} propose to rely on samples with high variance thus improve the convergence of classification tasks. 
%\cite{katharopoulos2017biased} proposed to use another machine learned sample 
%weight for importance sampling, however, the algorithm need to evaluate the 
%learned weight for a large portion of the whole dataset during each batch 
%generation. 
%
%To make sure that out model introduces minimal overhead to the training process 
%of the \boss, we design the pipeline that has the light weighted 
%\assistant train on CPU and thus its computation time can be hide 
%behind the training of \boss on GPU. We can further accelerate this 
%process by maintain a queue for training batches and conduct gradient update 
%of \assistant and \boss asynchronously. 
Our contributions in this paper are as follows. 
\begin{compactitem}
  \item We propose an \autoassist learning framework with an \assistant which 
    can shrink less informative instances and generate smart batches in an ROI aware sequence to the \boss to perform SGD 
    updates. We also propose a computational scheme so that 
    learning of the \boss and the \assistant are done asynchronously, which minimizes 
    the overhead introduced by the \assistant.
  \item We prove that even with biased stochastic gradient, the instance 
    shrinking method based on gradient magnitude can still 
    guarantee the $O(\frac{1}{k})$ convergence achieved by plain SGD under the strongly convex setting. 
  \item We empirically show that the proposed \autoassist framework leads to improved 
      convergence as well as test loss on various deep 
      learning applications, including image classification and neural machine translation.
\end{compactitem}

\begin{comment}
\section{Preliminary}
In this section we introduce our notations and describe importance sampling techniques 
and its variants. Importance sampling can improve the convergence speed by reducing gradient 
variance of SGD. Given training data $\{x_i,y_i\}_{i=0}^N$ and a deep learning 
model $f(*,\theta)$, we seek to minimize the empirical risk:
\begin{align}
  \theta^* = \argmin_{\theta}  \frac{1}{N}\sum_{i=0}^N L(f(x_i,\theta), y_i) \\
\end{align} 
where $L()$ is the loss function. At each training step, a mini-batch 
$B\subset [N]$ is drew from training set and $\theta$ is optimized with 
stochastic gradient:
\begin{align}
  \theta_{t+1} = \theta_t - \eta \alpha_i \nabla_{\theta_t} \frac{1}{B} 
  \sum_{i\in B} L(f(x_i, \theta_t), y_i) \label{importance-sampling}
\end{align} 
where batch $B$ is drew from $Multinomial(p_1,\ldots,p_N)$ and 
$\alpha_i=\frac{1}{Np_i}$ being the sample weight. For classic SGD, 
$p_i=\frac{1}{N}$ and $\alpha_i=1$ for $i\in[N]$.  
\end{comment}

\vspace{-0.2cm}
\section{Related Work}
Considerable research has been conducted to optimize the way data is presented to 
the optimizer for deep learning.
For example, curriculum learning~(CL)~\cite{bengio2009curriculum}, which presents 
easier instances to the model before hard ones, was shown to be beneficial to the overall convergence; however, 
prior knowledge of the training set is required to decide the curriculum. 
Self-paced learning~(SPL)~\cite{kumar2010self} infers the 
difficulty of instances with the corresponding loss value and then decreases
the sample probability of difficult instances. Self-paced Convolutional 
Networks~(SPCN)~\cite{li2017self} combines the SPL algorithm with the training 
of Convolutional Neural Networks to get rid of noisy data. However, the SPL 
type methods usually result in over emphasis of the easier instances and thus 
harm performance. Similar ideas have been developed when optimizing for other machine learning models.
In classical SVM models, methods have been proposed to ignore trivial instances by dimension shrinking in 
dual coordinate descent. This accelerates the convergence speed by saving many 
useless parameter updates. 
%Chang \text{et. al.} proposed Active-Bias that focus on high variance instances 
%rather than hard or easy examples.  

Importance sampling is another type of method that has been proposed to 
accelerate SGD convergence. In importance sampling methods, instances are 
sampled by its importance weights. %instead of presented with equal probability.
%Zhao \textit{et. al.}
\cite{zhao2015stochastic} proposed Iprox-SGD that uses importance sampling to 
achieve variance reduction. 
The optimal importance weight distribution to reduce the variance of the 
stochastic gradient is proved to be the gradient norm of the
sample~\cite{needell2014stochastic,zhao2015stochastic,alain2015variance}. 
Despite the variance reduction effect, importance sampling methods tend to 
introduce large computational overhead.
Before each stochastic step, the 
importance weights need to be updated for all instances which makes 
importance sampling methods infeasible for large datasets.
%Katharopoulos \textit{et. al.}
\cite{katharopoulos2017biased} proposed an importance sampling scheme for deep 
learning models; however, in order to reduce computation cost for evaluating 
importance scores, the proposed algorithm applied a sub-sampling technique, 
thus leading to reliance on outdated importance scores for training. 

There are also several recent methods that propose to train an 
attached network with the original one. ScreenerNet~\cite{kim2018screenernet} 
trains an attached neural network to learn a scalar weight for each training 
instance, while MentorNet~\cite{jiang2017mentornet} learns a data-driven curriculum 
that prevents the main network from over-fitting. Since the additional 
model is another deep neural network, these types of methods introduce 
substantial computational and memory overhead to the original training process. 

Unlike previous methods, our proposed \autoassist framework (1)~does not need 
prior knowledge of the task, (2)~is able to utilize the response of the latest deep model to 
pick informative training instances, and (3) can greatly reduce overhead 
through CPU/GPU asynchronous training.

\vspace{-0.2cm}
\section{Dual coordinate shrinking \textit{vs} primal instance shrinking}
\renewcommand{\R}{\mathcal{R}}
\newcommand{\E}{\mathbb{E}}

In this section, we motivate our methodology by introducing the shrinking algorithm in dual coordinate 
descent of support vector machines~(SVMs) and showing that it is the same as 
ignoring certain instances in primal stochastic gradient descent. 
\vspace{-0.2cm}
\subsection{Shrinking in dual coordinate descent}
Consider the SVM for data classification. Given a set 
of instance-label data pairs $\{\xb_i,y_i\}_{i\in[N]}$ where $\xb_i\in \R^{n}$ and 
$y_i\in\{+1,-1\}$, SVM tries to solve the following minimization problem:
\begin{align}\label{primal_svm}
  \min_{\wb} \frac{1}{2}\wb^\top \wb + C\sum_{i=1}^N l(\wb,\xb_i,y_i),
\end{align}
where $C>0$ and loss function $l$ defined as:
\begin{align}
  l(\wb,\xb_i,y_i) = \max(1-y_i \wb^\top \xb_i,0).
\end{align}
The form \eqref{primal_svm} is often referred to as the primal 
form of SVM. By Representer Theorem, the primal parameter can be written as:
\begin{align}
  \wb = \sum_{j=1}^N \alpha_j y_j \xb_j
\end{align}
for some $\alphab\in\R^{N}$, thus we can solve the following dual form instead:
\begin{align}\label{dual_svm}
   \min_{\alphab} &\>  f(\alphab) = \frac{1}{2}\alphab^\top \Qb \alphab - \eb^\top \alphab \\
   s.t. &\>  0\le \alpha_i \le C, \forall i \notag,
\end{align}
where $Q_{ij}=y_i y_j \xb_i^\top \xb_j$. The above dual form is usually solved by cyclic
coordinate descent; during training step $k$, the update rule for coordinate~$i$ is given by:
\begin{align}\label{dual_update}
  \alpha_i^{k+1} = \min\left(\max\left(\alpha_i^{k} - \frac{\nabla_i f(\alphab^k)}{Q_{ii}}, 
  0\right), C\right).
\end{align}
It is important to note that the $i$-th coordinate of the dual parameter 
$\alphab\in\R^N$ corresponds to instance 
$(\xb_i, y_i)$ in the primal form. In update rule~\eqref{dual_update}, because of 
the constraint on $\alphab$, it is likely that an $\alpha_i$ stays anchored at $0$ or $C$ 
for many iterations~(an instance far from the decision boundary). As the algorithm 
converges, more dimensions of $\alphab$ will be at the constrained 
boundary and thus lead to many redundant computations. Thus the algorithm can 
be made more efficient by shrinking the parameter dimension in dual 
space~\cite{hsieh2008dual}. Let $A\subseteq[N]$ be the set of dual 
coordinates to be ignored and $B=[N] \backslash A$ be the subset after removing $A$.
Then the shrunk dual problem is:
\begin{align}\label{dual_svm}
  \min_{\alphab_B} &\>  f(\alphab) = \frac{1}{2}\alphab_B^\top \Qb \alphab_B + ( 
  \Qb_{BA} \alphab_A - \eb_B)^\top \alphab_B \\
  s.t. &\>  0\le \alpha_i \le C, \forall i\in B. \notag
\end{align}
Solving the above shrunk dual form with coordinate descent is equivalent to solving 
the primal form~\eqref{primal_svm} with stochastic gradient descent without 
considering the subset of trivial instances $\{\xb_i,y_i\}_{i\in A}$. In many 
applications, the dual variable $\alphab$ is usually sparse since most data 
points are far from the decision boundary, and hence the savings from shrinking are considerable. 
%In practice people relax this method to not only hinge loss scenarios and 
%developed soft version of shrinking algorithm.
%In soft shrinking, we ignore the dimensions of $\alpha$ such that 
%$\alpha_i>C-\epsilon$ or $\alpha<\epsilon$ where the threshold $\epsilon$ 
%could be defined by a fraction of the mean loss:
%$$
%\epsilon = \frac{\beta}{N} \sum_{i=1}^N min(\alpha, C-\alpha)
%$$
%where $0 < \beta \ll 1$ controls the aggressiveness of the shrinking. This is 
%equal to stochastic gradient descent in the primal form while ignoring the 
%instances with small loss ({\color{red}BUGBUG}):
%$$
%\{(x_i,y_i) | l(w,x_i, y_i) < \frac{\beta}{N} \sum_{j=0}^N l(w, x_j, y_j) \}
%$$
%By Theorem 1 and 2 in~\cite{hsieh2008dual}, both the dual shrinking and dual soft 
%shrinking algorithm will globally converge to an optimal solution $\alpha^*$ 
%thus leading to an optimal primal problem solution 
%$w^* = \sum_{i=0}^N \alpha^*_i y_i x_i$.
\vspace{-0.2cm}
\subsection{Instance Shrinking in Stochastic Gradient Descent}

The above dual shrinking method can largely save time and space while dealing 
with large data. This motivates a similar shrinking 
method for solving the primal problem. Generally speaking, given dataset 
$\{\xb_i, y_i\}_{i\in[N]}$ and objective function $f(\wb, \xb,y)$ parameterized by $\wb$, we would 
like to solve the minimization problem:
\begin{align}\label{obj_general}
  \min_\wb \frac{1}{N}\sum_{i=1}^N f(\wb, \xb_i, y_i)
\end{align}
In many applications, stochastic gradient descent is often used when $N$ is very 
large. At each stochastic gradient step, an instance $(\xb_i,y_i)$ or a batch of 
instances $\{\xb_i,y_i\}_{i\in B}$
are sampled from the training data and a gradient descent step is conducted 
based on the stochastic gradient. However, presenting all data to the 
optimizer will distract the optimizer from focusing on the most important 
instances. We seek to build a smart batch 
generator that can select the most informative batches given the model 
condition, thus accelerating convergence to save training time.
Similar to the shrinking algorithm in the dual problem, we can 
apply a shrinking algorithm by ignoring trivial instances at the current training stage.
The criteria to decide whether an instance is trivial or not can be 
the objective function value $f(\wb,\xb_i, y_i)$ or the gradient magnitude 
$\nabla f(\wb,\xb_i,y_i)$. Specifically, a threshold
%gradient magnitude
$T$ can be set so that 
any instance that has gradient magnitude lower than $T$ is ignored:
\begin{align}\label{shrinking_grad}
  \wb_{k+1} = \wb_k - \eta \nabla f(\wb_k, \xb_i, y_i)\mathbf{I}(\|\nabla 
  f(\wb_k, \xb_i, y_i)\|\ge T).
\end{align}
\vspace{-0.2cm}
\subsection{Importance sampling and computational overhead}

Similar ideas have also appeared in importance sampling techniques, where the data 
are sampled with importance weights rather than uniformly. 
Per sample gradient norm is usually used as the importance weight and has been 
proved to be the optimal importance weight distribution~\cite{needell2014stochastic,zhao2015stochastic,alain2015variance}.
However, as the model changes after every parameter update, importance weights 
need to be updated, which is computationally prohibitive.
A pre-sampling technique is sometimes used to tackle this issue.
First a subset of data $C\subset [N]$~(with $|C|\ll N$) is uniformly sampled and importance 
scores are evaluated only on $C$. After 
training enough number of batches on $C$, another chunk is sampled and 
evaluated. This can reduce the computational overhead but may introduce new 
issues. Firstly, the importance weights are fixed once evaluated and 
may be outdated after parameter updates. 
In real applications with large data, the model can evolve substantially even 
within one epoch through the data.  
Secondly, substantial computational overhead is introduced even with the sub-sampling technique. 
As for the shrinking method, at each training step, only the objective function of the selected instances needs to be evaluated rather than the whole 
dataset. This still introduces certain amount of overhead when most instances 
can be ignored. However, as shown in Figure~\ref{svm_tests}, shrinking is able 
to have better performance compared to plain SGD in terms of the number of
parameter updates and have similar performance in terms of training time, 
whereas importance sampling method is not applicable due to its large overhead. 
In the case of SVM training, the computation saved by shrinking 
method is negligible due to the simplicity of the SVM model. However, when it comes to deep learning models, the 
reduction in computation cost is significant. In Section~\ref{autoassist}, we 
describe a deep learning training pipeline that is motivated by the shrinking 
method such that the computational overhead is negligible. 

\begin{figure*}
  \centering
  \includegraphics[width=0.246\textwidth]{\fig/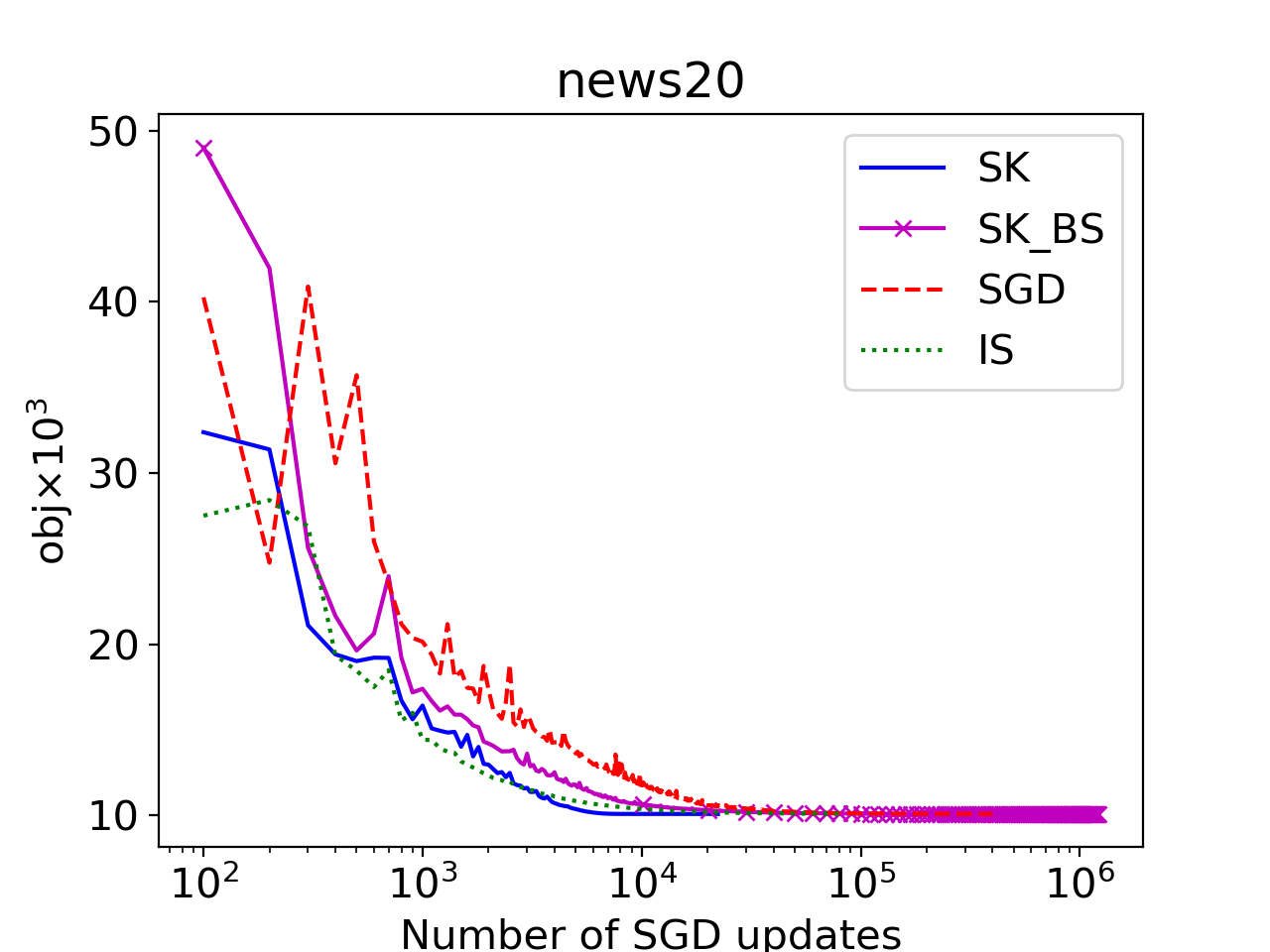}
  \includegraphics[width=0.246\textwidth]{\fig/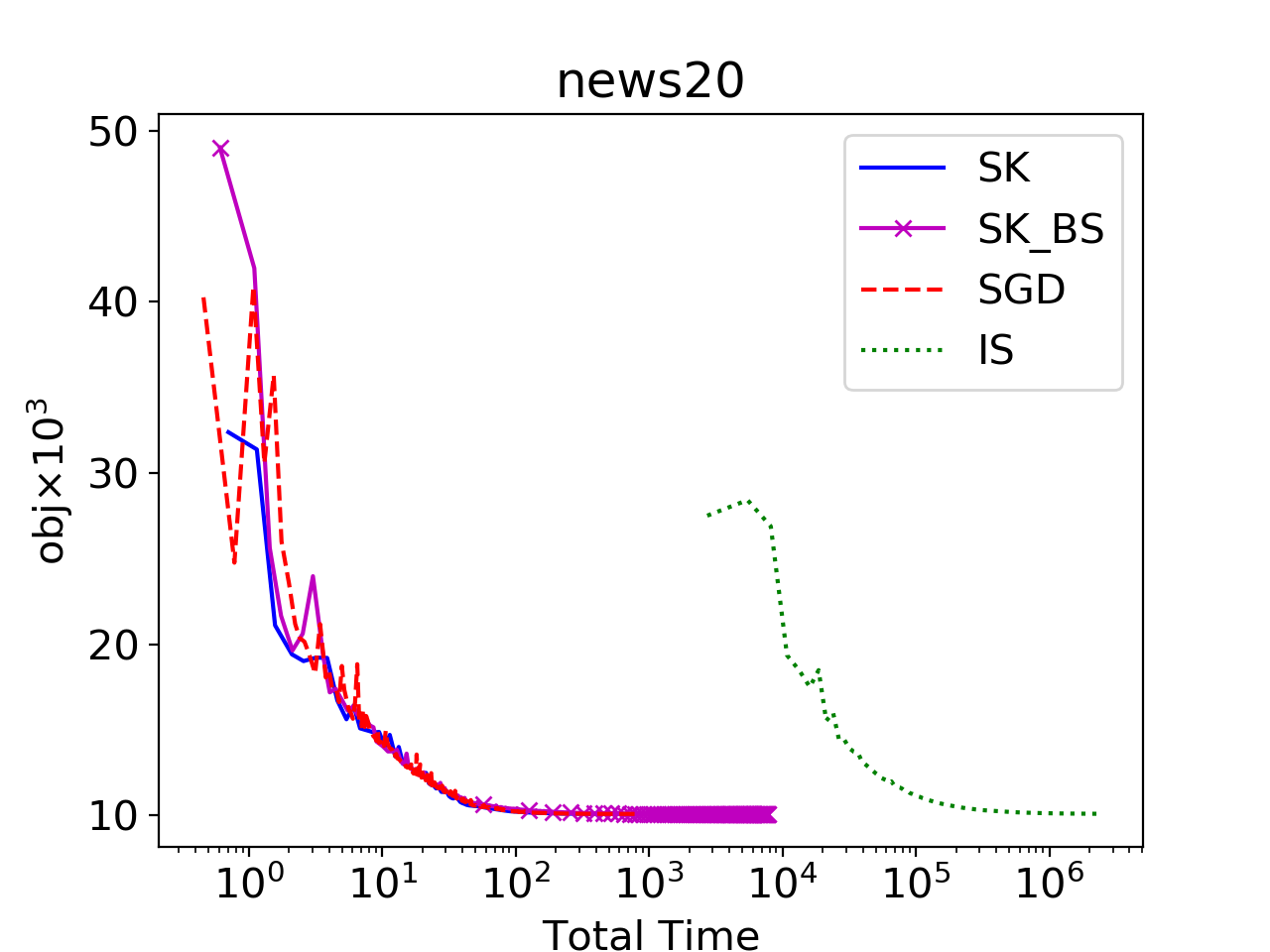}
  \includegraphics[width=0.246\textwidth]{\fig/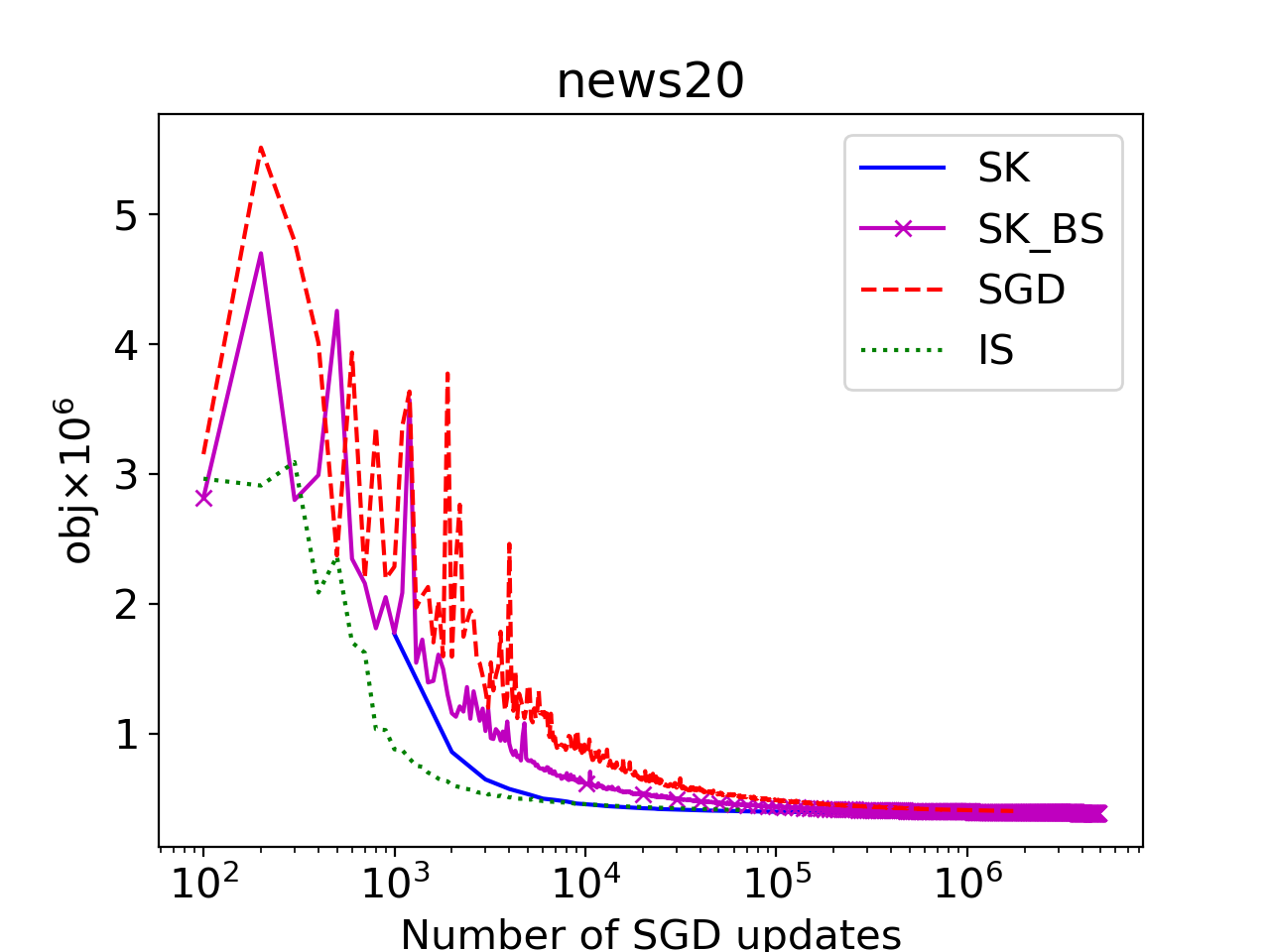}
  \includegraphics[width=0.246\textwidth]{\fig/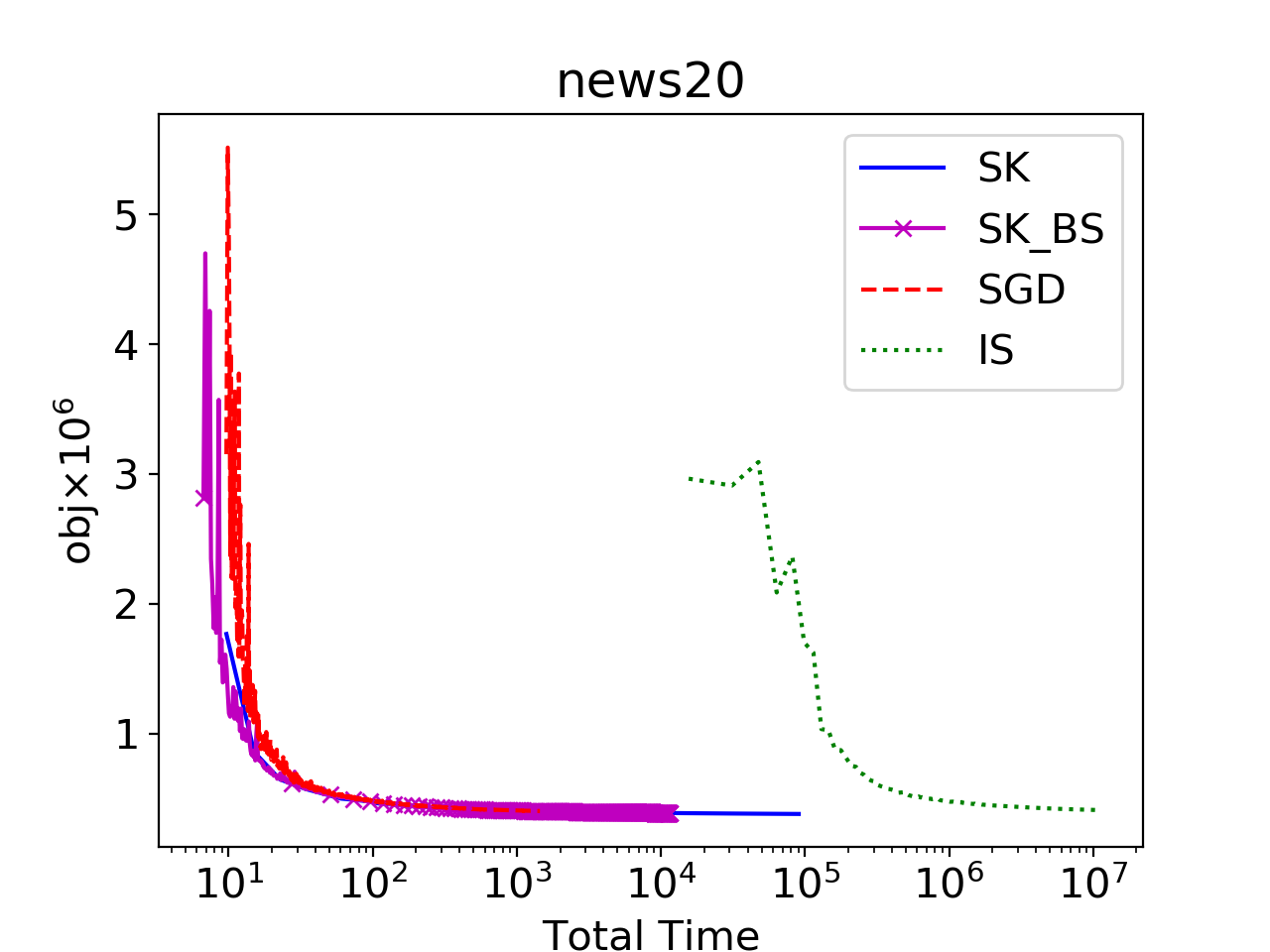}
  \caption{SVM on news20 and rcv1 datasets. In terms of the number of parameter 
    updates, the shrinking method~({\bf SK}), shrinking method with Bernoulli 
    sampling~({\bf SK$\_$BS}), and importance sampling~({\bf IS}) methods have 
comparable performance, and are better than plain SGD. However, because 
of large computational overhead, importance sampling methods are not 
applicable to large datasets. In these experiments, per sample gradient 
magnitude is used as the importance weight distribution for importance sampling 
method and objective function value is used as shrinking criterion.}
  \label{svm_tests}
\end{figure*}
\vspace{-0.2cm}
\subsection{Shrinking experiment under convex setting}

We test the shrinking algorithm on SVM binary classification tasks. The shrinking 
algorithm~({\bf SK}) evaluates the objective function value upon receiving 
every randomly sampled training instance and deciding whether to take a 
training step on it. We compare shrinking~({\bf SK}) with plain SGD and 
importance sampling~({\bf IS}) algorithms as well as shrinking 
algorithm with Bernuli sampling~({\bf SK$\_$BS}). In {\bf SK$\_$BS}, different 
from hard shrinking~({\bf SK}), one does not make immediate decision to ignore a certain 
instance but conduct a Bernuli sampling to make that decision; the detailed 
description can be found in Algorithm~1. The datasets 
are from libsvm~\cite{REF08a,CC01a} and we implemented 
Pegasos~\cite{shalev2011pegasos} algorithm with C++ as the baseline method. 
All shrinking methods and importance sampling method achieve faster 
convergence compared with baseline SGD. However, because of the overhead to update 
importance weights, importance sampling method takes much more time to conduct 
same number of parameter updates, which makes it impractical for large 
datasets. Both shrinking methods have comparable time cost as baseline SGD and 
still have advantage over baseline SGD w.r.t. training time.

\vspace{-0.2cm}
\section{Theoretical Analysis}

In this section, we establish the theoretical convergence of SGD with the
proposed shrinking operations.
It is easy to see that if certain instances are ignored during batch
generation process, the stochastic gradients on that batch would be biased.
Majority of the stochastic gradient methods try to build an
unbiased estimator of the full gradient. Vanilla SGD achieves this through
uniform sampling or random permutation.
Importance sampling
methods~\cite{katharopoulos2017biased,katharopoulos2018not} sample
training batches based on importance weights. To ensure an unbiased
stochastic gradient, an inverse weighted loss is typically used to replace the original
loss in many importance sampling approaches. 
However, it is not necessary for an unbiased gradient estimate
to be used to guarantee convergence of a stochastic gradient method.
Many recent stochastic gradient algorithms are designed without using a
unbiased gradient estimator, such as stochastic averaging
gradients~(SAG)~\cite{schmidt2017minimizing} or asymptotic biased SGD~\cite{tadic2017asymptotic}. 
Generally speaking, a SGD variant is able to converge as long as the difference between the
used estimated gradient and the true gradient converges to zero. In
Figure~\ref{svm_tests}, we have observed that SGD with our proposed
shrinking operations converges to the same objective function values as
vanilla SGD for an L2-regularized linear SVM problem. We now theoretically
prove that SGD with the proposed shrinking operation~\eqref{shrinking_grad} is able to converge at $O(1/k)$ rate for
$\mu$-strongly convex problems.
%Now, As a result, it is more flexible to design
% We empirically verified this by the experiments on SVM with hinge loss
%and l2 regularization. In fig~\ref{svm_tests}, shrinking algorithm converge to
%the same objective function value as plain SGD and importance sampling. Now we
%theoretically prove that shrinking method is able to converge even with biased
%stochastic gradient estimator.
In particular, with similar assumptions as used to
establish convergence of vanilla SGD~\cite{bottou2018optimization}, 
let us consider a function
$ F(\bw)$ that satisfies the following conditions:
\begin{property}\label{def1}
$F(\bw) := \frac{1}{N}\sum_{i=1}^N f_i(\bw)$ satisfies: 
\begin{itemize}
\item $F$ is $\mu$-strongly convex,
\item $\arg\min_{\bw} F(\bw) \in B_D = \cbr{\bw \mid \norm{\bw} \le D}$, and
\item $\norm{\nabla f_i\rbr{\bw}} \le G,\ \forall \bw \in B_D$
\end{itemize}
\end{property}
Note that the boundedness assumption is motivated from
Pegasos~\cite{shalev2011pegasos}, where a projection step is incorporated to
limit the set of admissible solutions to a ball of radius $1/\mu$.
%If we define the convergence by the reduction of the distance of the
%parameter $\wb$ from the optimal parameter vector $\wb^*$, with $\mu$-strongly
%convex $f(\wb)$, SGD ensures convergence:
%$$
%\E(\|\wb_k-\wb^*\|^2) \le \frac{\max\left( 4D^2, \frac{G^2}{\mu^2} \right)}{k}
%$$
%where $f: B_D(0) \rightarrow \R$, $B_D(0)=\{\wb\in\R^d | \|\wb\|\le D\}$ and
%$G=\max_{\wb\in B_D(0)}(\|\nabla f(\wb)\|)$. The assumption of bounded parameter
%norm is motivated from the algorithm Pegasos~\cite{shalev2011pegasos}, where a
%projection step is incorporated to limit the set of admissible solutions to
%the ball of radius $\frac{1}{\mu}$.
In Theorem~\ref{thm:shrink_converge_grad} (with similar assumptions as used to
establish convergence of vanilla SGD~\cite{bottou2018optimization}), we show that SGD with our proposed
shrinking techniques converges for $\mu$-strongly convex problem with an $O(1/k)$ convergence
rate.

\begin{theorem}
	\label{thm:shrink_converge_grad}
  Suppose function $F$ satisfies Property~\ref{def1}. 
  For the stochastic gradient descent with the update rule
  $\bw_{k+1} = \bw_{k} - \eta_k \bg_k$, where
  step size $\eta_k=\frac{1}{\mu k}$ and stochastic gradient $\bg_k$ is based
  on a uniformly randomly selected instance index $i$:
  \begin{align}
    \bg_k &= \nabla f_i(\wb_k) - \epsilonb_k,\ i\sim \uniform{1,2,\ldots,N}\\
  \epsilonb_k &= \nabla f_i (\wb_k) \mathbf{I}(\|\nabla f_i(\wb_k)\|\le
  \frac{G}{k})
  \end{align}
  we have
  \begin{align}
    \EE(\|\wb_k-\wb^*\|^2) \le \frac{\max\left( 4D^2, \frac{G^2}{\mu^2} + \frac{4DG}{\mu}\right)}{k}
    \label{eq:shrink-conv}
  \end{align}
  where $w^*$ is the optimal solution.
\end{theorem}

%\begin{proof}[Proof of Theorem~\ref{thm:shrink_converge_grad}]
\proof[Proof of Theorem~\ref{thm:shrink_converge_grad}]
  From strong convexity, we have:
  \begin{align}
  F(\wb^*) - F(\wb_k) \ge \langle \nabla F(\wb_k), \wb^* - \wb_k \rangle +
  \frac{\mu}{2}\|\wb_k - \wb^*\|^2 \notag \\ %\label{strong_convex_1}
  F(\wb_k) - F(\wb^*) \ge \langle \nabla F(\wb^*), \wb_k - \wb^* \rangle +
  \frac{\mu}{2}\|\wb_k - \wb^*\|^2. \notag
  \end{align}
  Adding the above inequalities gives:
  \begin{align}%\label{ineq1}
    \langle \nabla F(\wb_k) - \nabla F(\wb^*), \wb_k - \wb^*\rangle &\ge \mu\| \wb_k - \wb^*\|^2 \notag\\
 \Rightarrow\quad\quad\quad \langle \nabla F(\wb_k), \wb_k - \wb^*\rangle &\ge \mu\| \wb_k - \wb^*\|^2\notag\\
 \Rightarrow\ \quad \langle \E(\nabla f_i(\wb_k)), \wb_k - \wb^*\rangle &\ge \mu\| \wb_k - \wb^*\|^2 \label{ineq1}
  \end{align}
  Next, we have:
  \begin{align}
     & \E(\|\wb_{k+1}-\wb^*\|^2)
     = \E(\|\wb_{k} - \eta_k \gb_k - \wb^*\|^2) \notag\\
     =& \E(\|\wb_k-\wb^*\|^2) - 2\eta_k \E\langle \gb_k, \wb_k - \wb^*\rangle + \eta_k^2
     \E(\|\gb_k\|^2) \notag\\
     \le& \E(\|\wb_k-\wb^*\|^2) - 2\eta_k \E\langle \nabla f_i(\wb_k), \wb_k - \wb^*\rangle
      \notag\\
     & + \eta_k^2 G^2 + 2\eta_k \E\langle \epsilonb_k, \wb_k - \wb^*\rangle \label{next_distance}
  \end{align}
  The last term from \eqref{next_distance} can be upper bounded as follows:
  \begin{align}
     & 2\eta_k \E\langle \epsilonb_k, \wb_k - \wb^*\rangle \notag \\
     =& 2\eta_k \E\left(\langle \nabla f_i(\wb_k), \wb_k - \wb^*\rangle \mathbf{I}\left(
     \|\nabla f_i(\wb_k)\| \le \frac{G}{k}\right)\right)  \notag\\
     \le & 2\eta_k \E\left( \|\nabla f_i(\wb_k)\| \|\wb_k - \wb^*\| \mathbf{I}\left(
     \|\nabla f_i(\wb_k)\| \le \frac{G}{k}\right)\right)  \notag\\
     \le & 2\eta_k \frac{2DG}{k}\label{ineq4}
  \end{align}
  Substituting \eqref{ineq1} and \eqref{ineq4} into \eqref{next_distance}, we have:
  \begin{align}
     & \E(\|\wb_{k+1}-\wb^*\|^2)
     = \E(\|\wb_{k} - \eta_k \gb_k - \wb^*\|^2) \notag\\
   \le& \E(\|\wb_k-\wb^*\|^2) (1-\frac{2}{k}) + \left(\frac{4DG}{\mu} + \frac{G^2}{\mu^2}\right) \frac{1}{k^2}
  \end{align}
  Let $L = \max(\|\wb_1-\wb^*\|, \frac{G^2+4DG\mu}{\mu^2})$, the $O(1/k)$
  convergence \eqref{eq:shrink-conv} can be established by induction. When $k=1$, the result holds:
  \begin{align}
    \E(\|\wb_1-\wb^*\|^2) \le \frac{\max\left( 4D^2, \frac{G^2}{\mu^2}+\frac{4DG}{\mu} \right)}{1}
  \end{align}
  Suppose the result holds for $k$, then for $k+1$:
  \begin{align}
     & \E(\|\wb_{k+1}-\wb^*\|^2)
     = \E(\|\wb_{k+1} - \eta_k \gb_k - \wb^*\|^2) \notag\\
     \le& \left(1-\frac{2}{k}\right)\frac{L}{k} + L\frac{1}{k^2} = \frac{k-1}{k^2}L \le \frac{L}{k+1}
  \end{align}
\QED

\begin{comment}
\begin{theorem}
	\label{thm:shrink_converge_loss}
  Let $f(w)=\frac{1}{N}\sum_{i=1}^N f_i(w)$ and $f_i(w):B_D(0)\rightarrow\R$ being $\mu$
  strongly convex. $\|\nabla f_i(w)\|\le G, \|\nabla^2 f_i(w)\|\le M, \forall i$. At each
  stochastic gradient step $k$, step size $\eta_k=\frac{1}{\mu k}$ and
  stochastic gradient $g_k$ being:
  \begin{align}
  g_k &= \nabla f_i(w_k) - \epsilon_k, i\sim Uniform(1,2,\ldots,N)\\
  \epsilon_k &= \nabla f_i (w_k) \mathbf{I}( f_i(w_k) - f_i(w^*)\le
  \frac{\mu G^2}{2kM^2})
  \end{align}
  then
  \begin{align}
    \E(\|w_k-w^*\|^2) \le \frac{\max\left( \|w_1-w^*\|^2, \frac{G^2}{4}(\frac{1}{M^2}+\frac{1}{\mu^2}) \right)}{k}
  \end{align}
  where $w^*$ is the optimal solution and $f_i(w^*)=0$.
\end{theorem}
The proof of Theorem~\ref{thm:shrink_converge_loss} is similar to that of
Theorem~\ref{thm:shrink_converge_grad} and can be found in the appendix.
\end{comment}
Theorem~\ref{thm:shrink_converge_grad}
indicates that in the strongly convex case, shrinking based on gradient
magnitude is able to achieve the same convergence as SGD even though the
stochastic gradient is biased. For many ML models, computation of objective
function value is cheaper than computation of gradients. Also, for $\mu$-strongly
convex $F(\wb)$ which is $M$-Lipschitz smooth we have:
\begin{align}
\|\nabla F(\wb)-\nabla F(\wb^*)\|^2&=\|\nabla F(\wb)\|^2 \le M^2\|\wb-\wb^*\|^2\notag\\
F(\wb)-F(\wb^*) &\ge \frac{\mu}{2}\|\wb-\wb^*\|^2 \notag
\end{align}
Thus, gradient
magnitude and objective function value is the same in terms of measuring instance
triviality:
$$
\mu (F(\wb)-F(\wb^*)) \le \frac{1}{2}\|\nabla F(\wb)\|^2 \le \frac{M^2}{\mu}
(F(\wb) - F(\wb^*))
$$
where the left inequality is the Polyak-Lojasiewicz inequality. Due to these reasons, in practice we
can usually use the loss function value as the shrinking criterion.

\vspace{-0.2cm}
\section{AutoAssist: Learning with an Assistant}\label{autoassist}
%We propose a way to make fast and up to date decision on whether to use a 
%given training instance: the \autoassist model. In this section we explain the 
%proposed model within deep learning framework and introduce a CPU/GPU parallel 
%training structure that can greatly reduce training overhead.
The goal of instance shrinking is to reduce the model training 
time, especially for tasks with large number of instances and complex models 
such as deep learning. Although the shrinking method is able to greatly reduce 
the number of updates needed to converge, the 
improvement is not that obvious in terms of training time. The reason is that we need almost
the same amount of computation to decide whether to ignore an instance as to conduct the gradient 
update. Thus the reduced time for parameter updating is compensated by the 
overhead introduced by the decision procedure for shrinking. To reduce such overhead, we 
need to have a cheaper way to make shrinking decisions. Specifically for deep 
learning, using a light weight model for instance shrinking can save us a 
lot of computation. 
In this section, we propose a training framework, named \autoassist, 
that trains a light weight model to make the shrinking decision for deep 
learning models.
\vspace{-0.2cm}
\subsection{\assistant for instance selection}

In many machine learning applications, 
it is observed that data that follows a certain pattern can be 
better handled than others. For example, in image 
classification tasks, at a certain training phase, a 
shirt and a pair of jeans may be well distinguished, but a pair of shorts may be 
confused with a pair of jeans. Also, in machine translation tasks, 
sentences containing certain ambiguous tokens may not be translated well, 
while those with precise meanings may be handled very well. Many such patterns 
can be learned through very simple models, such as a shallow convolutional 
network or a bag-of-words classifier. 
Motivated by these observations, we propose to attach a light-weight model~(\assistant) to the major 
deep neural network~(\boss) to assist the \boss with selecting informative training instances. 

The traditional deep learning training pipeline includes two major parts: the batch generator~(BG) 
and the forward propagation~(FP) / backward propagation~(BP) machine. A vanilla 
batch generator iterates through the whole dataset which is randomly permuted 
after each epoch. In the \autoassist training framework, a \boss~(FP/BP 
machine) and an \assistant~(machine learned batch generator) work together. 
The \assistant is constructed such that it can (1) choose training batches cleverly to 
boost \boss training while (2) introducing low overhead.
To resolve (1) we let the \assistant learn from the performance of the \boss on 
different examples and generate most informative training batch at each 
training stage. Also, because of the design of \autoassist, it is possible to 
use a CPU/GPU parallel learning scheme that minimizes \assistant overhead to \boss.
\vspace{-0.2cm}
\subsection{Learning to generate smart training batches}
%To dynamically generate mini-batches for training, we need to evaluate 
%instance importance based on \boss' ability to handle them.
The ideal \assistant would evolve with the \boss and make accurate selections 
based on the latest behavior of the \boss.
At the beginning of training, the 
\assistant presents easy examples to \boss in order to get better convergence. 
As boss converges, the easy instances become less and less informative and the 
\assistant gradually decreases the presentation ratio of easy instances and 
increases the difficulty of instances in the batch according to the ability of \boss to handle 
them. 
Specifically, we design the \assistant network~($g(\cdot)$) to be a classifier parameterized by $\phib$ 
that tries to predict the difficulty of instances by minimizing: 
\begin{align}
  \min_{\phib}  \frac{1}{N}\sum_{i=0}^N \hat{L}( g(\xb_i, \yb_i, \phib), z_i) \label{sec-loss}
\end{align} 
where $\hat{L}$ is the cross entropy loss and $z_i$ is the binary label indicating if this is a trivial instance. 
One possible definition of $z_i$ is:
\[
z_i=\left\{\begin{array}{ll}
1                  &, \;\;\; f(\xb_i,\yb_i, \thetab) > T\\
0                  &, \;\;\; {\mbox{\rm otherwise}}.
\end{array}\right.
\]  
where $T$ is the threshold for instance shrinking and 
$f(\xb,\yb,\thetab)$ is the objective function value of the \boss network parameterized by $\thetab$.
By jointly learning \boss and \assistant, we can 
guarantee that the \assistant network is trained on the latest labels, i.e. the $z_i$'s 
generated by the latest \boss model $f(\xb_i, \yb_i,\thetab)$. 
The training batch $B\subset [N]$ is then
generated via a series of Bernoulli samplings with a smoothing term~$\gamma$. 
\begin{algorithm}
  \label{alg:batch}
  \begin{itemize}
    \item {\bf Input:} Training dataset $D=\{\xb_i, \yb_i\}_{i=1}^N$, base 
      probability $\gamma$
    \item {\bf Output:} sampled batch index $B\subset[N]$ 
    \item {\bf Initialize:} $B=[]$
    \item While $B$.size() $<$ batch$\_$size:
      \begin{itemize}
        \item $idx \sim \texttt{uniformInt}(N)$
        \item $c \sim \texttt{uniform}(0,1)$
        \item If $c < \gamma$: 
          \begin{itemize}
            \item B.append($idx$)
          \end{itemize}
        \item Else:
          \begin{itemize}
            \item $\hat{c}=\frac{c-\gamma}{1-\gamma}$ // another \texttt{uniform}(0,1) variable
            \item If $\hat{c}<g(\xb_i, \yb_i,\phib)$:
            \begin{itemize}
              \item B.append($idx$)
            \end{itemize}
          \end{itemize}
      \end{itemize}
    \item {\bf Return} $B$
  \end{itemize}
  \caption{Assistant.sample$\_$batch}
\end{algorithm}

Noting that $p_i$ and $\frac{p_i-\gamma}{1-\gamma}$ are independent variables,
we only need to evaluate $g_i$ upon receiving a random sampled 
index $i$. In practice, index $i$ loops over the randomly shuffled training data 
index list rather than sampled uniformly. Thus a well-trained \assistant will 
save us a lot of computation time by skipping trivial instances, thus accelerating convergence.
Usually the \assistant model is chosen to be light weight and may share 
knowledge from the \boss. For example, in image classification tasks, the 
\assistant may be constructed as a convolution layer followed by a pooling 
layer and a single linear layer, where the convolution layer shares parameter s
with the \boss. 
\vspace{-0.2cm}
\subsection{Asynchronous joint learning with CPU/GPU}
\begin{figure*}[t]
  \centering
  \includegraphics[width=0.45\textwidth]{\fig/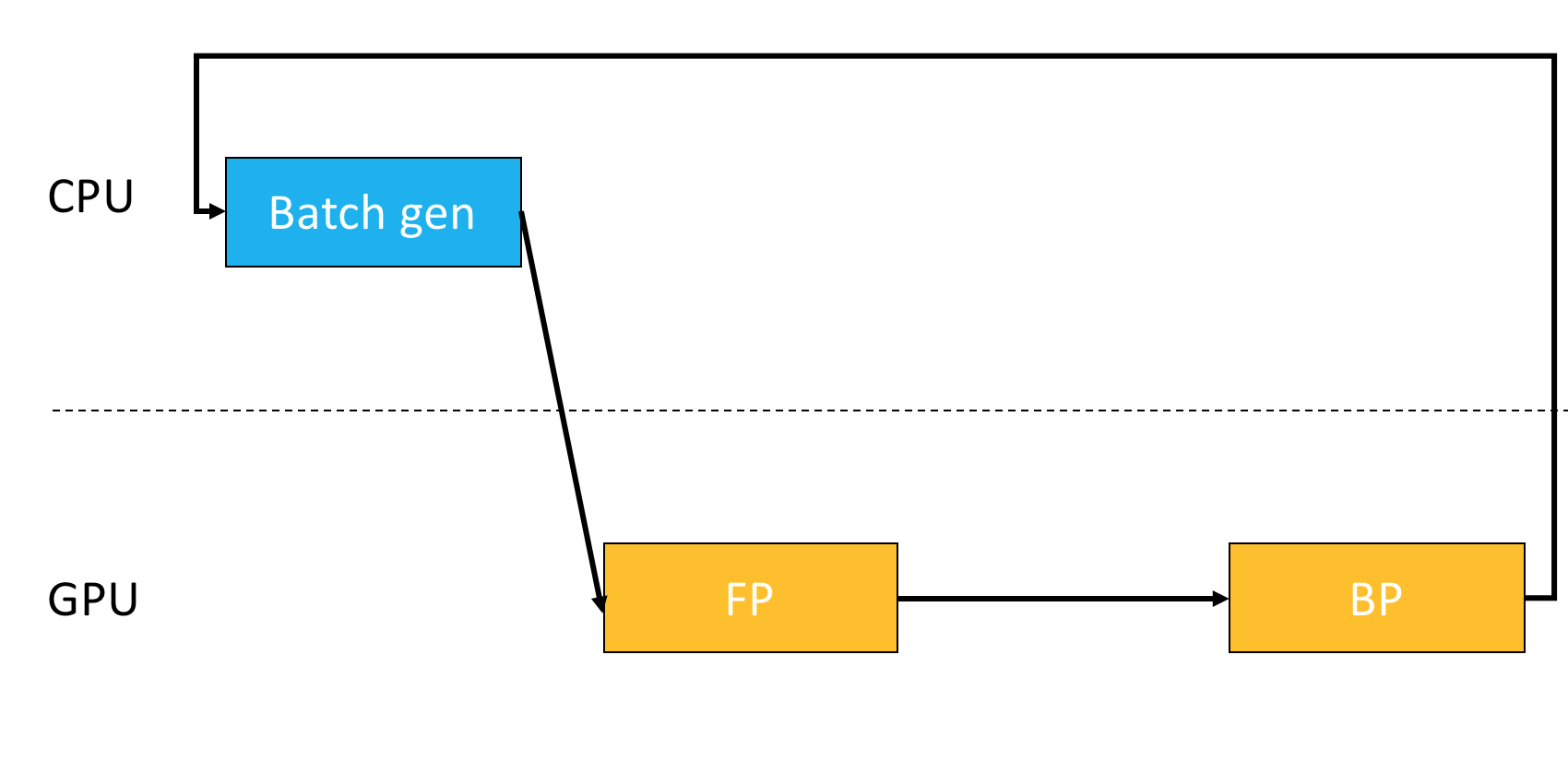}
  \includegraphics[width=0.45\textwidth]{\fig/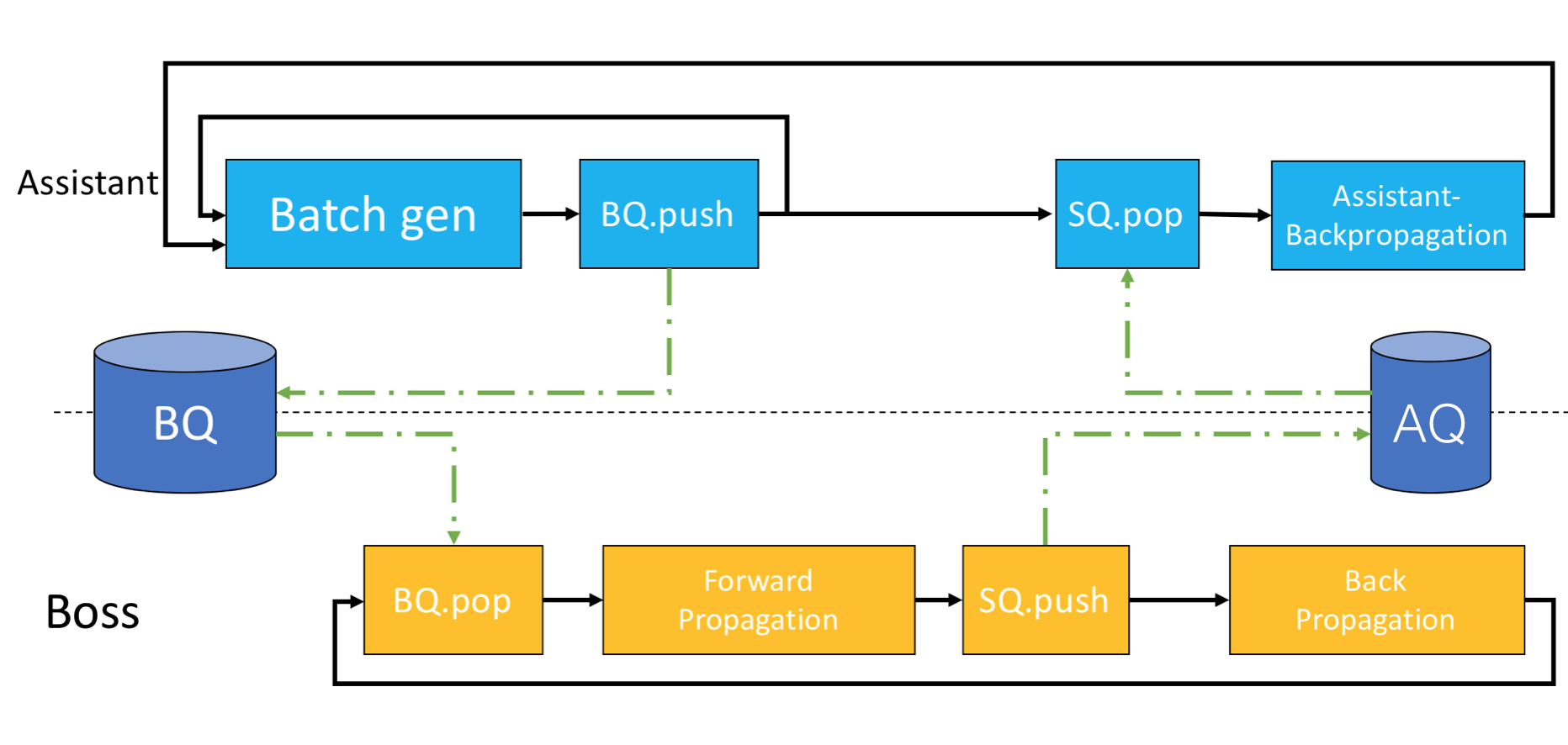}
\label{pipeline-fig}
\caption{The computation scheme of vanilla training on the left wastes CPU/GPU 
  cycles, while the accelerated 
  training process with \autoassist trains both the \boss and the \assistant asynchronously.} 
\end{figure*}
Another concern is the concept that \textit{a good \assistant will best utilize the 
time schedule of \boss}. 
That is we want the \assistant network to introduce as 
little overhead to the training of the \boss network as possible. 
In traditional mini-batch training of deep networks, the batch generation on CPU 
and training on GPU are done sequentially. Thus at least one of the CPU/GPU is 
idle at each moment during training. In our \autoassist framework, we propose a 
training algorithm where both the CPU and GPU work asynchronously on their 
jobs. This reduces the overhead of batch generation or \assistant training to 
GPU time and could potentially be generalized to multi-CPU setting. 
Specifically, we maintain two queues of mini-batches:
\begin{align}
  \text{BossQueue}&=\{\{\xb_i,\yb_i\},\ldots\}\\ 
  \text{AssistantQueue}&=\{\{i,f(\xb_i, \yb_i, \thetab)\},\ldots\}
\end{align}
During each mini-batch training step, \boss obtains a batch of instances 
$\{\xb_i,\yb_i\}_{i\in B}$ from 
{\bf BossQueue}~(BQ) and conducts the forward propagation on $B$ to obtain 
losses $\{f(\xb_i, \yb_i, \thetab)\}$. Then the \boss pushes the 
index-loss pairs $\{i, f(\xb_i, \yb_i, \thetab)\}_{i\in B}$ to {\bf AssistantQueue}~(AQ). 
\textit{Assistant} on the other hand, trains on items from {\bf AssistantQueue} 
and sample training batch and push to {\bf BossQueue}. 
In most cases, the training step of \boss is much more expensive than that of 
\assistant thus it is easy for \assistant
to keep {\bf BossQueue} non-empty at all times. The only overhead on \boss's 
timeline is pushing the loss information to AQ, which is minimal work 
considering that each instance only consists of two scalars~(index, loss).

\vspace{-0.2cm}
\section{Experimental results}
In this section we present the experimental results. We tested the \autoassist 
model on two different tasks: image classification and neural machine translation. 
All experiments are done on Tesla V100 GPUs with implementation in 
PyTorch~\cite{paszke2017automatic}. 

\subsection{Image classification}
We tested our models on MNIST~\cite{lecun-mnisthandwrittendigit-2010} hand 
written digit dataset, 
FashionMNIST~\cite{xiao2017fashion} fashion item image dataset, Extended MNIST~\cite{cohen2017emnist} and 
CIFAR10~\cite{krizhevsky2009learning} datasets. A logistic regression model is used as the \assistant network and 
ResNet~\cite{he2016deep} as the \boss network.  
%Auto-assistV1 is constructed as a standalone linear classifier that feeds the \boss with examples that itself makes wrong classification while 
\autoassist is constructed to optimize the binary classification objective~\eqref{sec-loss}.
Adam~\cite{kingma2014adam} is used as the optimizer for all models. 
  \begin{figure}[t]
  \centering
  \includegraphics[width=0.42\textwidth]{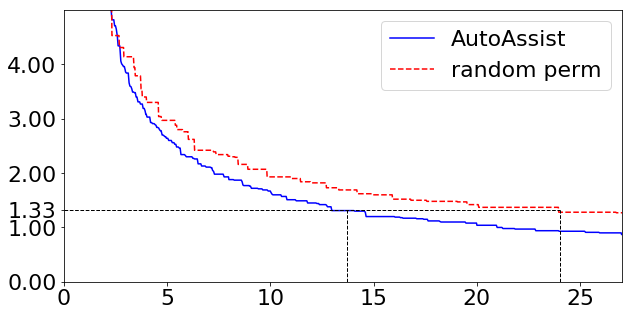} 
  \includegraphics[width=0.42\textwidth]{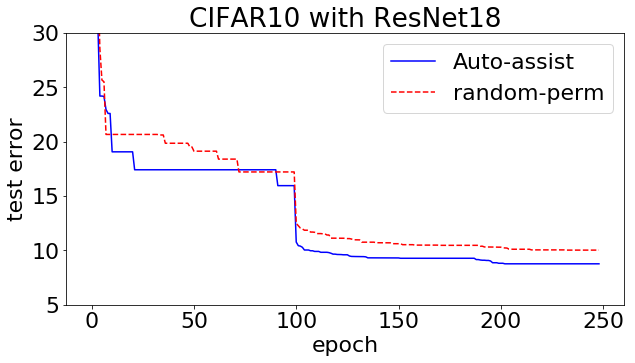} 
  \label{image-epoch}
  \caption{Test accuracy on MNIST and CIFAR10 datasets --- to reach error 
  1.33 on MNIST,  \autoassist needs only about half the number of epochs as random permutation.}
  \end{figure}

Adding Auto-assist is able to improve the final test accuracy on MNIST and CIFAR10 
datasets and enables faster convergence. On the MNIST dataset, to achieve $98.67\%$ test accuracy, the baseline 
approach takes 24 epochs through the entire data, while it takes Auto-assist just 13.7 epochs 
to achieve similar accuracy. On other image 
classification datasets, \autoassist is also able to improve the final test 
accuracy. 
\begin{table}[h]
\centering
\begin{tabular}{l*{3}{c}r}
 {\bf Data}      & {\bf random-perm } &  {\bf \textit{\autoassist}}\\
\hline
MNIST    &98.67 & {\bf 99.17} \\
FashionMNIST &90.68 & {\bf 90.99} \\
ExtendedMNIST & 87.21 & {\bf 87.28} \\
CIFAR10    & 90.04 & {\bf 90.23}
\end{tabular}
\caption{Accuracy with/without \assistant.}
\end{table}

\subsection{Machine Translation}
For machine translation tasks, we tested our \autoassist model with two 
translation datasets: the English-German Image 
Description dataset~\cite{elliott-EtAl:2016:VL16} and WMT~2014 English-German 
dataset. The \boss model is chosen to be the transformer 
model~\cite{vaswani2017attention} and the \assistant is chosen to be a bag of 
word classifier. For the smaller dataset Multi30K, which consists of around 30k 
language pairs, we train the transformer \textit{base} model on a single GPU 
machine and used batch size of 64 sentences. Besides the vanilla random 
shuffling batch generation method, we also compared with the 
self-paced-learning~(SPL) algorithm~\cite{kumar2010self,li2017self} and choose 
the best pace step $q(t)$ over multiple tests. 

\begin{figure}[t]
\centering
\includegraphics[width=0.42\textwidth]{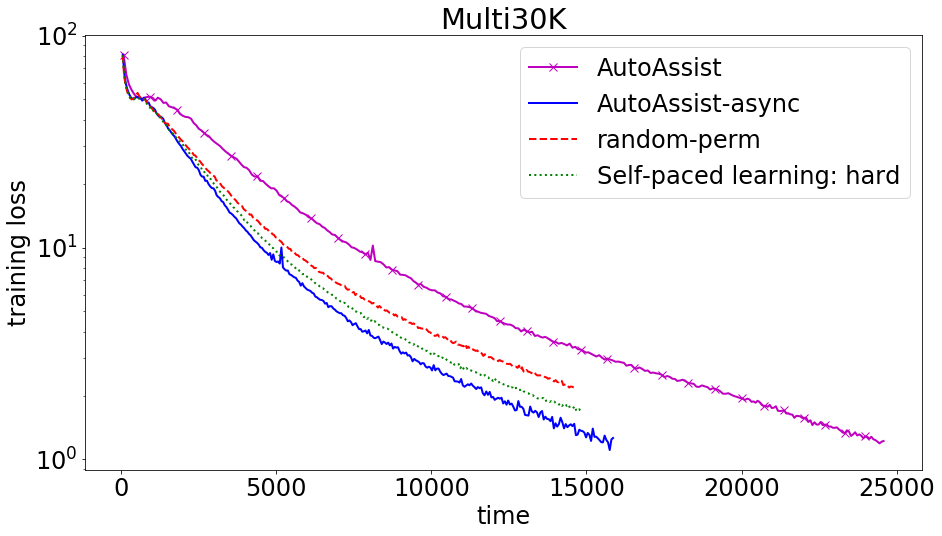}
\includegraphics[width=0.42\textwidth]{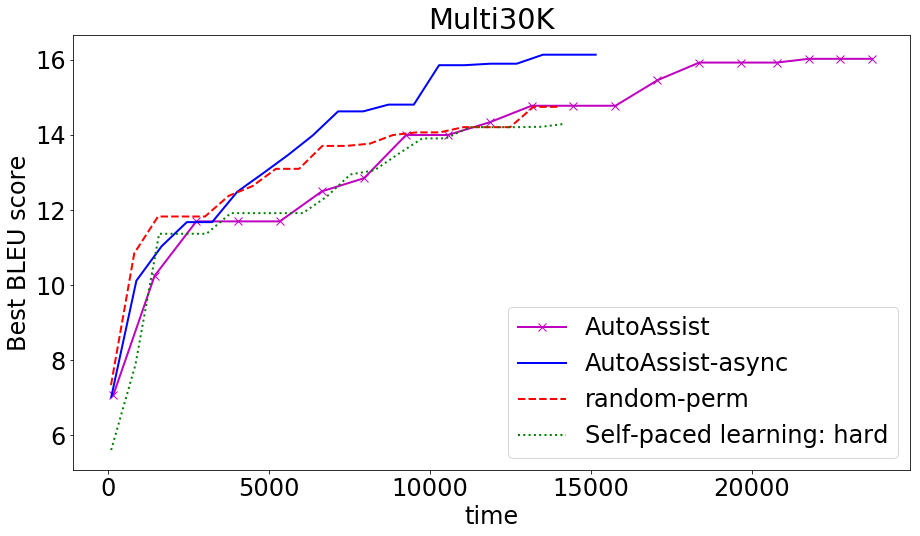} 
\label{nmt-fig1}
\caption{The training error and BLEU score on Multi30k dataset. \autoassist 
with CPU/GPU asynchronous training converges faster and achieves higher BLEU score.}
\end{figure}

Table~\ref{table-time} shows the time required to complete one epoch on the Multi30k 
dataset. With the proposed asynchronous training pipeline, the time overhead 
introduced by \autoassist is only $5\%$ while the sequential implementation has $61\%$ overhead.
\begin{table}[h]
\centering
\begin{tabular}{l*{3}{c}r}
{\bf Model}    &  sentence-pairs/second\\
\hline
{\bf baseline}  & 617.3   \\
{\bf AutoAssist}  & 382.7   \\
{\bf AutoAssist-async}  & 586.0   
\end{tabular}
\caption{Number of sentence pairs trained per second. CPU/GPU 
asynchronous training is able to decrease $92\%$ of the time overhead.}
\label{table-time}
\end{table}

On the WMT 2014 English to German dataset consisting of about 4.5 million
sentence pairs, we trained the transformer model with the \textit{base} setting 
in~\cite{vaswani2017attention} and constructed source vocabulary and target 
vocabulary of size 40k and 43k respectively. Both vanilla and \autoassist 
models are trained on 8 Tesla V100 GPUs with around 40k tokens per training 
batch~(5k tokens per batch per GPU). For the vanilla model, batches are pre-generated and randomly shuffled at the
beginning of each epoch. For the \autoassist model, data are split into 8 chunks 
with the same number of tokens and 8 \assistant models are trained simultaneously to 
generate batches for each GPU. Each \assistant will thus be trained on a 
subset of the data and generate training batch from that chunk of data. We 
generate translation results with beam size of 4 and length penalty of 0.6. We 
are able to obtain 27.1 BLEU score with both the vanilla and \autoassist model.s 
The \autoassist is able to achieve higher BLEU score with fewer 
tokens seen. For example, the vanilla model reaches BLEU score of 26 after 
training on 1.82 billion tokens while the \autoassist only needs 1.18 billion 
tokens. This means that the instances the \assistant decides to ignore have 
little contribution to the convergence and thus can be ignored with no harm.
\begin{figure}[H]
\centering
\includegraphics[width=0.42\textwidth]{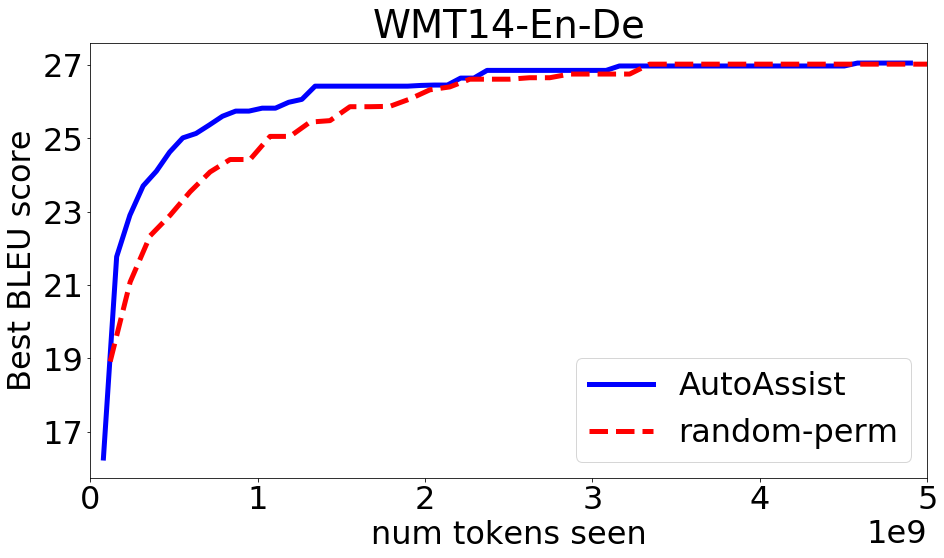} 
\label{nmt-fig1}
\caption{BLEU score on WMT14 En-De dataset}
\end{figure}

\begin{table}[H]
\centering
\begin{tabular}{l*{3}{c}r}
 {\bf Model}      & {\bf vanilla } &  {\bf \textit{Auto-assist}}\\
\hline
avr $\#$ tokens per epoch    & 12.1M  & 7.87M \\
\hline
BLEU reach 26 & $1.82\times 10^9$ & {\bf 1.18$\times10^9$} \\
BLEU reach 27 & $3.22\times 10^9$ & {\bf 3.01$\times10^9$} \\
\hline
\end{tabular}
\caption{Number of tokens needed to reach different BLEU score level. \autoassist 
is able to save $35\%$ tokens to reach 26 and $6\%$ to reach 27.}
\end{table}

\section{Conclusions}

In this paper, motivated by the dual coordinate shrinking method in SVM dual coordinate 
descent, we propose a training framework to accelerate deep learning model training.
The proposed framework, Auto-assist, jointly 
trains a batch generator~(\assistant) along with the main deep learning model~(\boss). 
The \assistant model conducts primal instance shrinking to get rid of trivial instances 
during training and can automatically adjust the criteria based on the ability 
of the \boss. In the strongly convex setting, even though our method leads to 
biased gradients, we should that stochastic gradient with instance shrinking has 
$O(\frac{1}{k})$ convergence, which is the same as plain SGD. 
To maximize the training efficiency of CPU/GPU cycles, we let the \assistant learn from the output of the \boss 
in an asynchronous parallel scheme.
We further propose a method to reduce the computational overhead by training \assistant and \boss 
asynchronously on CPU/GPU.
Experiments show that both convergence and accuracy could be improved through 
introducing our \assistant model. The CPU/GPU asynchronous 
training pipeline is able to reduce the overhead to less than $10\%$ 
compared with the sequential pipeline.

\newpage
\bibliographystyle{plain}
\bibliography{references}

\begin{thebibliography}{10}

\bibitem{alain2015variance}
Guillaume Alain, Alex Lamb, Chinnadhurai Sankar, Aaron Courville, and Yoshua
  Bengio.
\newblock Variance reduction in sgd by distributed importance sampling.
\newblock {\em arXiv preprint arXiv:1511.06481}, 2015.

\bibitem{bengio2015scheduled}
Samy Bengio, Oriol Vinyals, Navdeep Jaitly, and Noam Shazeer.
\newblock Scheduled sampling for sequence prediction with recurrent neural
  networks.
\newblock In {\em Advances in Neural Information Processing Systems}, pages
  1171--1179, 2015.

\bibitem{bengio2009curriculum}
Yoshua Bengio, J{\'e}r{\^o}me Louradour, Ronan Collobert, and Jason Weston.
\newblock Curriculum learning.
\newblock In {\em Proceedings of the 26th annual international conference on
  machine learning}, pages 41--48. ACM, 2009.

\bibitem{bottou2018optimization}
L{\'e}on Bottou, Frank~E Curtis, and Jorge Nocedal.
\newblock Optimization methods for large-scale machine learning.
\newblock {\em SIAM Review}, 60(2):223--311, 2018.

\bibitem{CC01a}
Chih-Chung Chang and Chih-Jen Lin.
\newblock {LIBSVM}: A library for support vector machines.
\newblock {\em ACM Transactions on Intelligent Systems and Technology},
  2:27:1--27:27, 2011.
\newblock Software available at \url{http://www.csie.ntu.edu.tw/~cjlin/libsvm}.

\bibitem{chang2017active}
Haw-Shiuan Chang, Erik Learned-Miller, and Andrew McCallum.
\newblock Active bias: Training more accurate neural networks by emphasizing
  high variance samples.
\newblock In {\em Advances in Neural Information Processing Systems}, pages
  1002--1012, 2017.

\bibitem{cohen2017emnist}
Gregory Cohen, Saeed Afshar, Jonathan Tapson, and Andr{\'e} van Schaik.
\newblock {EMNIST}: an extension of {MNIST} to handwritten letters.
\newblock {\em arXiv preprint arXiv:1702.05373}, 2017.

\bibitem{collobert2008unified}
Ronan Collobert and Jason Weston.
\newblock A unified architecture for natural language processing: Deep neural
  networks with multitask learning.
\newblock In {\em Proceedings of the 25th international conference on Machine
  learning}, pages 160--167. ACM, 2008.

\bibitem{elliott-EtAl:2016:VL16}
D.~{Elliott}, S.~{Frank}, K.~{Sima'an}, and L.~{Specia}.
\newblock Multi30k: Multilingual english-german image descriptions.
\newblock pages 70--74, 2016.

\bibitem{REF08a}
Rong-En Fan, Kai-Wei Chang, Cho-Jui Hsieh, Xiang-Rui Wang, and Chih-Jen Lin.
\newblock {LIBLINEAR}: A library for large linear classification.
\newblock {\em Journal of Machine Learning Research}, 9:1871--1874, 2008.

\bibitem{he2016deep}
Kaiming He, Xiangyu Zhang, Shaoqing Ren, and Jian Sun.
\newblock Deep residual learning for image recognition.
\newblock In {\em Proceedings of the IEEE conference on computer vision and
  pattern recognition}, pages 770--778, 2016.

\bibitem{hsieh2008dual}
Cho-Jui Hsieh, Kai-Wei Chang, Chih-Jen Lin, S~Sathiya Keerthi, and
  Sellamanickam Sundararajan.
\newblock A dual coordinate descent method for large-scale linear svm.
\newblock In {\em Proceedings of the 25th international conference on Machine
  learning}, pages 408--415. ACM, 2008.

\bibitem{jiang2015self}
Lu~Jiang, Deyu Meng, Qian Zhao, Shiguang Shan, and Alexander~G Hauptmann.
\newblock Self-paced curriculum learning.
\newblock In {\em AAAI.}, page Vol. 2. No. 5.4., 2015.

\bibitem{jiang2017mentornet}
Lu~Jiang, Zhengyuan Zhou, Thomas Leung, Li-Jia Li, and Li~Fei-Fei.
\newblock Mentornet: Learning data-driven curriculum for very deep neural
  networks on corrupted labels.
\newblock {\em arXiv preprint arXiv:1712.05055}, 2017.

\bibitem{katharopoulos2017biased}
Angelos Katharopoulos and Fran{\c{c}}ois Fleuret.
\newblock Biased importance sampling for deep neural network training.
\newblock {\em arXiv preprint arXiv:1706.00043}, 2017.

\bibitem{katharopoulos2018not}
Angelos Katharopoulos and Fran{\c{c}}ois Fleuret.
\newblock Not all samples are created equal: Deep learning with importance
  sampling.
\newblock {\em arXiv preprint arXiv:1803.00942}, 2018.

\bibitem{kim2018screenernet}
Tae-Hoon Kim and Jonghyun Choi.
\newblock Screenernet: Learning curriculum for neural networks.
\newblock {\em arXiv preprint arXiv:1801.00904}, 2018.

\bibitem{kingma2014adam}
Diederik~P Kingma and Jimmy Ba.
\newblock Adam: A method for stochastic optimization.
\newblock {\em arXiv preprint arXiv:1412.6980}, 2014.

\bibitem{krizhevsky2009learning}
Alex Krizhevsky.
\newblock Learning multiple layers of features from tiny images.
\newblock Technical report, Citeseer, 2009.

\bibitem{kumar2010self}
M~Pawan Kumar, Benjamin Packer, and Daphne Koller.
\newblock Self-paced learning for latent variable models.
\newblock In {\em Advances in Neural Information Processing Systems}, pages
  1189--1197, 2010.

\bibitem{langkvist2014review}
Martin L{\"a}ngkvist, Lars Karlsson, and Amy Loutfi.
\newblock A review of unsupervised feature learning and deep learning for
  time-series modeling.
\newblock {\em Pattern Recognition Letters}, 42:11--24, 2014.

\bibitem{lecun-mnisthandwrittendigit-2010}
Yann LeCun and Corinna Cortes.
\newblock {MNIST} handwritten digit database.
\newblock 2010.

\bibitem{li2017self}
Hao Li and Maoguo Gong.
\newblock Self-paced convolutional neural networks.
\newblock In {\em Proceedings of the International Joint Conference on
  Artificial Intelligence}, 2017.

\bibitem{needell2014stochastic}
Deanna Needell, Rachel Ward, and Nati Srebro.
\newblock Stochastic gradient descent, weighted sampling, and the randomized
  kaczmarz algorithm.
\newblock In {\em Advances in Neural Information Processing Systems}, pages
  1017--1025, 2014.

\bibitem{paszke2017automatic}
Adam Paszke, Sam Gross, Soumith Chintala, Gregory Chanan, Edward Yang, Zachary
  DeVito, Zeming Lin, Alban Desmaison, Luca Antiga, and Adam Lerer.
\newblock Automatic differentiation in {P}y{T}orch.
\newblock 2017.

\bibitem{schmidt2017minimizing}
Mark Schmidt, Nicolas Le~Roux, and Francis Bach.
\newblock Minimizing finite sums with the stochastic average gradient.
\newblock {\em Mathematical Programming}, 162(1-2):83--112, 2017.

\bibitem{shalev2011pegasos}
Shai Shalev-Shwartz, Yoram Singer, Nathan Srebro, and Andrew Cotter.
\newblock Pegasos: Primal estimated sub-gradient solver for svm.
\newblock {\em Mathematical programming}, 127(1):3--30, 2011.

\bibitem{tadic2017asymptotic}
Vladislav~B Tadi{\'c}, Arnaud Doucet, et~al.
\newblock Asymptotic bias of stochastic gradient search.
\newblock {\em The Annals of Applied Probability}, 27(6):3255--3304, 2017.

\bibitem{vaswani2017attention}
Ashish Vaswani, Noam Shazeer, Niki Parmar, Jakob Uszkoreit, Llion Jones,
  Aidan~N Gomez, {\L}ukasz Kaiser, and Illia Polosukhin.
\newblock Attention is all you need.
\newblock In {\em Advances in Neural Information Processing Systems}, pages
  5998--6008, 2017.

\bibitem{xiao2017fashion}
Han Xiao, Kashif Rasul, and Roland Vollgraf.
\newblock Fashion-{MNIST}: a novel image dataset for benchmarking machine
  learning algorithms.
\newblock {\em arXiv preprint arXiv:1708.07747}, 2017.

\bibitem{zhao2015stochastic}
Peilin Zhao and Tong Zhang.
\newblock Stochastic optimization with importance sampling for regularized loss
  minimization.
\newblock In {\em international conference on machine learning}, pages 1--9,
  2015.

\end{thebibliography}

\newpage
\onecolumn
\begin{appendices}	
\section{Algorithms for CPU/GPU parallelization}

\begin{algorithm}[h]
\label{alg:assistant}
\begin{itemize}
  \item {\bf Input:} Training dataset $D=\{\xb_i, \yb_i\}_{i=1}^N$ 
\item {\bf Initialize:} BossQueue, AssistantQueue, Assistant
  \item While {\bf True}:
    \begin{itemize}
      \item If BossQueue.size()$<k$:
        \begin{itemize}
          \item B = Assistant.sample$\_$batch()
          \item BossQueue.enqueue(B)
        \end{itemize}
      \item Else if not AssistantQueue.empty():
        \begin{itemize}
          \item \Mb = AssistantQueue.pop()
          \item grad = Assistant.gradient(\Mb)
          \item Assistant.update(grad)
        \end{itemize}
    \end{itemize}
\end{itemize}
\caption{algorithm for \textit{assistant}~(CPU)}
\end{algorithm}

\begin{algorithm}[h]
\label{alg:boss}
\begin{itemize}
  \item {\bf Input:}
\item {\bf Initialize:} Boss
  \item While {\bf True}:
    \begin{itemize}
      \item If not BossQueue.empty():
        \begin{itemize}
          \item B = BossQueue.pop()
          \item M = Boss.Forward(B)
          \item AssistantQueue.enqueu(M)
          \item grad = Boss.Backward(M)
          \item Boss.update(grad)
        \end{itemize}
    \end{itemize}
\end{itemize}
\caption{algorithm for \boss~(GPU)}
\end{algorithm}

\end{appendices}

\end{document}